\pdfoutput=1
\documentclass[11pt]{article}
\usepackage[margin=1in]{geometry}
\usepackage{amsmath,amssymb,amsthm}
\usepackage{graphicx}
\usepackage{booktabs}
\usepackage{hyperref}
\newtheorem{theorem}{Theorem}
\newtheorem{lemma}{Lemma}
\newtheorem{proposition}{Proposition}
\newtheorem{corollary}{Corollary}
\theoremstyle{definition}
\newtheorem{definition}{Definition}

\theoremstyle{remark}
\newtheorem{remark}{Remark}

\title{\textbf{On the Necessity of a Liquid Substrate for Mesh Intelligence}}
\author{Hongwei Xu\\ SYM.BOT\\ \texttt{hongwei@sym.bot}}
\date{June 2026}

\begin{document}
\maketitle

\begin{abstract}
\noindent
A mesh of sovereign agents has no center: no shared clock, no shared model, and no coordinator to gather
data or retrain. Its competence rests on each agent folding the projections its peers emit into a single
internal state, online, from observations that arrive at irregular, unscheduled times, on a substrate whose
weights it cannot retrain. Any one of these constraints is tractable on its own; folding optimally under all
three at once is not. We ask what such a substrate must be, and prove two necessary conditions from one model
of a self-evolving latent observed at irregular, exogenous times. Because the latent changes, its optimal
estimator is time-varying: an adaptive timescale is necessary, and every fixed-gain filter is strictly
suboptimal. And because arrivals are clock-free, the optimal estimate depends on the
elapsed gap between them, which no gap-blind network recovers at any width or depth. This second
condition is \emph{capacity-independent}: scale cannot substitute for the missing dependence. The two
conditions intersect in the continuous-time \emph{liquid} class. An LSTM satisfies the first, a fixed
continuous-time filter the second, and a multi-timescale liquid network both. Synthetic experiments confirm
each: the network attains the timescale, and the separation is computed exactly. The characterization is necessary, not sufficient, and binds fixed-weight substrates: a network
free to retrain reaches the class by other means. Proved per agent, the necessity binds every agent of a
mesh, a structural condition on mesh intelligence.
\end{abstract}

\section{Introduction}
\label{sec:intro}
In a \emph{mesh}, sovereign agents infer and learn with no center. Each agent keeps a private, evolving
state. It \emph{admits} only what it finds relevant in what its peers \emph{emit}, and emits only typed
projections of that state, never the state itself. No agent and no coordinator holds the whole: the
intelligence lies in what each agent admits from its peers and integrates into its own state. We call this
\emph{mesh intelligence}, and it rests entirely on that per-agent integration.

Three constraints make that integration hard, and a mesh imposes all three at once. The observations arrive
at irregular times that no agent can schedule. The quantity each describes keeps moving, so the stream is
\emph{non-stationary} and any fixed summary goes stale. And the agent must operate \emph{online}, on a
substrate whose weights it cannot retrain, because a live mesh has no training loop, no dataset, and no
objective to optimize. Each constraint alone is tractable; their conjunction is what makes optimal folding hard.
Because arrivals are clocked, the elapsed gap between them shapes each
update\footnote{Information-theoretically: under exogenous (state-independent) sampling the elapsed gap is
non-informative about the latent's \emph{value} $s^\star$ ($I(\Delta;s^\star)=0$~\cite{shannon}), yet it
constrains the prior covariance, hence the gain: the covariance channel Thm~\ref{thm:sep} shows is necessary
information no gap-blind estimator can forgo (the gap sharpens the posterior given noisy observations while
saying nothing about where the latent is). $I(\Delta;s^\star)>0$ would require informative, state-modulated
sampling, which we do not assume.}; because the target keeps moving, the substrate must both retain past
structure and follow new change.

Coupling has two levels, and only one is ours. A \emph{content} gate decides \emph{which} of a peer's
projections to admit~\cite{svaf}; that is a separate problem. The \emph{temporal} substrate decides
\emph{how} the admitted observations accumulate into the agent's state, and no prior work settles what it
must be. We study it in the regime a live mesh forces: weights fixed once and never retrained, with only the
state evolving as data arrives, the reservoir or Liquid State Machine view~\cite{lsm,esn}. This is the mesh's
operating condition, not a modeling convenience: an agent cannot know its future peers, the latent it tracks
is never observed, and no coordinator gathers data or ships new weights. The question is then sharp:
\emph{what must a fixed-weight substrate be to integrate a non-stationary, irregularly timed stream
optimally?}

We answer with two necessities, both read from one model of a self-evolving latent observed at irregular
times (\S\ref{sec:necessity}). First, because the latent moves, the optimal estimator is time-varying, so the
substrate needs an \emph{adaptive} timescale; a fixed-gain filter is strictly suboptimal
(Prop.~\ref{prop:necessity}; Lemma~\ref{lem:floor} gives the explicit one-pole case). Second, because the
same latent is seen at clock-free times, the optimal estimate depends on the elapsed gap between arrivals, and
\emph{no} gap-blind network matches it at any width or depth (Thm~\ref{thm:sep}). This second result is
\emph{capacity-independent}: added size cannot close the gap. Together the two properties define the
\emph{liquid} class~\cite{cfc,ltc}: an LSTM has the first, a fixed continuous-time filter the second, and a
multi-timescale liquid network has both (Fig.~\ref{fig:cells}). We confirm the adaptive-timescale necessity
against fixed-timescale and naive-adaptive baselines, and compute the gap-awareness separation exactly
(\S\ref{sec:exp-liquid}).

\begin{figure}[t]
\centering
\includegraphics[width=\linewidth]{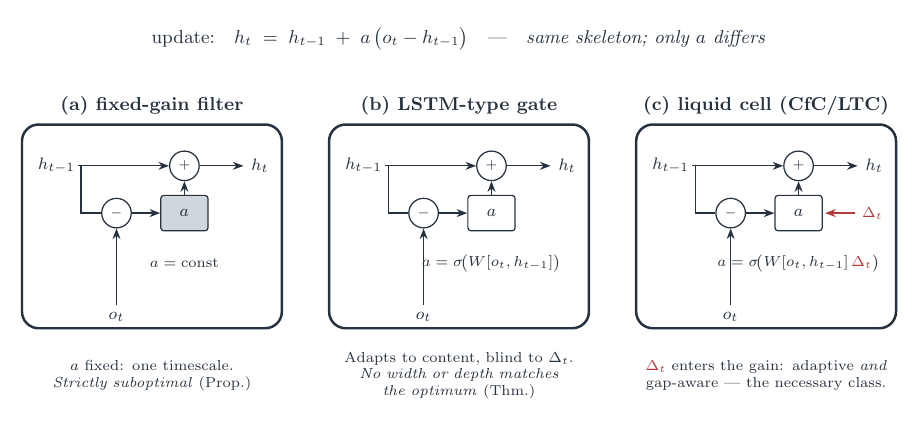}
\caption{\textbf{The two necessities in one cell.} Every substrate forms the same update
$h_t=h_{t-1}+a\,(o_t-h_{t-1})$; only the gain $a$ differs. \textbf{(a)}~A fixed gain is one
timescale: strictly suboptimal (Prop.~\ref{prop:necessity}). \textbf{(b)}~An LSTM-type gate
$a=\sigma(W[o_t,h_{t-1}])$ (a deliberate reduction; the full cell is richer but gap-blind all the same) adapts
to content but is indexed by \emph{step}, not elapsed time: it keeps state
through \emph{steps}, never sees the gap $\Delta_t$, and so matches the optimum at \emph{no} width or depth
(Thm.~\ref{thm:sep}). \textbf{(c)}~A liquid cell (CfC/LTC) admits $\Delta_t$ into the gain,
$a=\sigma(W[o_t,h_{t-1}]\,\Delta_t)$, adaptive \emph{and} gap-aware: the necessary class. Capacity is the
wrong axis: long step-memory does not encode elapsed duration.}
\label{fig:cells}
\end{figure}

Both necessities are proved for a single agent, but a mesh places every agent under the same conditions, so
the result binds each one. The conditions also cannot be evaded within a mesh: no agent can force its peers
onto a regular clock, the one case in which the gap stops mattering, and with no center to correct a
gap-blind agent the shortfall is present at every node. A liquid substrate is therefore required at every
agent. This is a structural condition on mesh intelligence rather than a property of any single filter,
obtained by reading the single-agent theorems across the agents of a mesh.

This necessity is one layer of a larger characterization. A mesh couples its agents on two orthogonal axes.
The companion Mesh Inference~\cite{meshinf} treats the \emph{spatial} axis, with the latent held static:
whether a center-free population can recover the collective answer at all. This paper treats the
\emph{temporal} axis, for a single agent: what its substrate must be to keep its own state current as the
latent evolves. With the content gate that decides which observations a receiver admits~\cite{svaf}, the
three name the layers a mesh agent runs on: \emph{which} to admit, \emph{whether} the collective can recover
the answer, and \emph{how} to track it in time. We characterize the last, and take the other two as the
companions' domain.

The result is \emph{scoped, not universal}. It states what optimal integration requires \emph{among
fixed-weight substrates}. A network free to retrain, such as GRU-D or Time-LSTM, can reach both properties by
learning, and a fully trainable learner is not constrained at all. We also do not claim that a liquid
substrate \emph{attains} the optimum, only that nothing outside the liquid class can; whether it does is the
empirical question of \S\ref{sec:exp-liquid}.

The ingredients are classical: time-varying-gain filtering and change
detection~\cite{kalman,haykin,basseville,bocpd}, continuous-time and liquid networks~\cite{cfc,ltc,neuralode},
and reservoir computing~\cite{lsm,esn}. What is new is the pairing: two necessities drawn from one model, the
second a capacity-independent separation, whose intersection \emph{defines} the liquid class for fixed-weight
online integration, with the scope drawn honestly.

\section{Related work}
\label{sec:related}

\paragraph{Irregularly-sampled time series.}
The standard route to irregular sampling feeds the elapsed gap to a discrete recurrence: GRU-D decays the
hidden state by the elapsed gap~\cite{grud}, Time-LSTM gates on it~\cite{timelstm}, and ODE-RNN /
latent-ODE models evolve the state by a learned ODE between observations~\cite{odernn}. These \emph{attain}
gap-awareness, and by Thm~\ref{thm:sep} they must to be optimal, but they do so by \emph{training} weights to
the gap. Our separation concerns what a \emph{fixed-weight} substrate can do without that training, where a
continuous-time cell is the native realization (Cor.~\ref{cor:native}).

\paragraph{Fixed-weight vs.\ trained substrates.}
The fixed-weight premise also separates our setting from distributed learning, which jointly optimizes a
shared model: federated averaging exchanges weights~\cite{fedavg}, decentralized SGD averages models with
neighbors~\cite{dsgd}, and multi-agent RL centralizes training~\cite{maddpg}. The substrate here is frozen
during operation (\S\ref{sec:model}): no weights are optimized, online or off, so there is nothing to share.

\paragraph{Continuous-time, liquid, and adaptive-timescale networks.}
The substrate is a closed-form continuous-time network~\cite{cfc}, the closed-form approximation to
the liquid time-constant ODE~\cite{ltc} in the neural-ODE family~\cite{neuralode}. Unlike a fixed-rate
recurrence, its time constants are input-dependent and multi-timescale, and that is what
\S\ref{sec:necessity} turns on: integrating a self-evolving latent \emph{requires} an adaptive timescale,
the estimation-theoretic reason a time-varying-gain filter dominates a fixed-gain one~\cite{kalman,haykin}
and change must be detected to be followed~\cite{basseville}. We do not claim the specific closed-form
network is the \emph{unique} such substrate; any adaptive continuous-time network satisfies the necessity.
What we add is to make the liquid (input-dependent, multi-timescale) property load-bearing, where prior
treatments leave it incidental. From reservoir computing, the Liquid State Machine~\cite{lsm} and echo-state
network~\cite{esn} lineage, we borrow only the fixed-substrate idea (a fixed continuous-time recurrence
carrying computation in its dynamics), not its offline-trained readout; we keep $\rho$ fixed too, so the only
adaptation is of the online state, not the weights (Remark~\ref{rem:fixed}).

\paragraph{Deep state-space models.}
Structured state-space models, HiPPO~\cite{hippo}, S4~\cite{s4}, and the selective-scan Mamba~\cite{mamba},
are recurrences derived from a continuous-time model and discretized with a step $\Delta$, and that step is
where our characterization bites. The same $\Delta$ can play two distinct roles: it yields an \emph{adaptive
timescale} (L1) when it is input-dependent, but \emph{elapsed-gap dependence} (L2) only when it \emph{is} the
gap. Mamba's selection makes $\Delta$ a function of \emph{content}, not of elapsed time
($\Delta=\mathrm{softplus}(\mathrm{Linear}(x))$); its scalar reduction (\cite{mamba}, Thm.~1) is the
content-gated recurrence $h_t=(1{-}g_t)h_{t-1}+g_t x_t$ with $g_t=\sigma(\mathrm{Linear}(x_t))$, precisely
the adaptive-but-gap-blind cell of Fig.~\ref{fig:cells}(b). On an irregularly-timed stream it is therefore
gap-blind, and Thm~\ref{thm:sep} binds it: even trained, no width or depth reaches the optimum until $\Delta$
is fed the elapsed gap. A frozen SSM \emph{joins} the liquid class precisely by carrying the gap in its
discretization, the continuous-time route of Cor.~\ref{cor:native}, a different choice of $\Delta$ than
content selection. Their weights are trained where our necessity is fixed-weight; the result then tells an
SSM designer the one structural property a frozen, deployed model must keep under irregular sampling:
$\Delta$ must carry the gap, not only the content. The same holds for a frozen Transformer: an LLM run on an
irregularly-timed stream is step-indexed, hence gap-blind, so the separation binds it at any scale; it needs
its estimate to depend on the gap, supplied structurally by a liquid substrate or extrinsically by feeding
the gap as input (Cor.~\ref{cor:native}).

\section{Setup: the model and the substrate}
\label{sec:model}
This section formalizes the setting and the substrate: the mesh (Def.~\ref{def:mesh}), the single-agent
integration problem each of its agents faces by construction (Def.~\ref{def:model}), and the fixed-weight
continuous-time substrate that integrates the resulting asynchronous stream into one internal state. The
substrate is a fixed-weight mechanism: what moves is the state, never the weights. Section~\ref{sec:necessity}
then shows what that substrate \emph{must} be.

\begin{definition}[Mesh; mesh intelligence]
\label{def:mesh}
A \emph{mesh} is a set of sovereign agents with no coordinator and no shared model. Each holds a private,
evolving cognitive state, \emph{emits} typed projections of it on its own clock (never the state itself),
and \emph{admits}, through a content gate, the projections it finds relevant from others, integrating
them into one internal state. \emph{Mesh intelligence} is the inference and learning the collective performs
from these admissions alone, with no center. Three properties then hold for \emph{each} agent by
construction: its admitted stream arrives at times the emitters set and it cannot schedule
(\emph{exogenous}); the cognitive states those projections describe keep evolving (\emph{non-stationary}); and
its integration substrate is fixed during operation, since no coordinator gathers data, sets an objective, or
retrains it.
\end{definition}

An agent thus estimates an evolving latent from a stream of observations arriving at irregular times: in a
mesh, the admitted projections of its peers. Our results concern the integration of whatever stream results;
let $\mathcal S(t)$ denote it up to time $t$. The three properties of Def.~\ref{def:mesh} are exactly the
hypotheses of the model below, so the necessities of \S\ref{sec:necessity} hold for \emph{each agent} of a
mesh by construction. The theorems are single-agent; the mesh is the setting that forces their conditions
rather than assuming them. That is why a liquid substrate is necessary \emph{for mesh intelligence}, and
equally why the result stands on its own.

\begin{definition}[Latent and observation model]
\label{def:model}
A single \emph{family} of models underlies the necessities of \S\ref{sec:necessity}. The latent $s(t)$ is
\emph{self-evolving}: the uncertainty accumulated over an elapsed interval of length $\Delta$,
$G(\Delta):=\operatorname{Var}\!\big(s(t{+}\Delta)-s(t)\big)$, is \emph{strictly increasing} in $\Delta$ with
$G(0)=0$. Two realizations recur, and we name whichever makes a mechanism most transparent: a
\emph{change-point} process (piecewise-stationary, change rate $\lambda>0$: stable for stretches, with
changes at unknown times; over a gap $\Delta$ a change has occurred with probability $1-e^{-\lambda\Delta}$,
so $G$ rises from $0$ toward the inter-level variance), and a \emph{Brownian drift} with $G(\Delta)=q\Delta$.
Observations arrive at irregular times $t_1<t_2<\cdots$ drawn \emph{exogenously} (independently of the latent
path and the noise), with gaps $\Delta_k=t_k-t_{k-1}$, as $o_k=s(t_k)+\varepsilon_k$, $\varepsilon_k$
zero-mean, variance $\sigma^2$. The two legs lean on complementary facets: Prop.~\ref{prop:necessity} uses
the \emph{change-point} realization specifically (regime alternation requires $\lambda>0$, making the
optimal gain time-varying); Thm~\ref{thm:sep} uses only $G$ strictly increasing under irregular sampling
(the elapsed gap \emph{sets} the accumulated uncertainty, making the optimal gain gap-dependent).
\end{definition}

\paragraph{The continuous-time substrate.}
The agent carries an internal \emph{state} $h\in\mathbb{R}^n$, its evolving estimate, and updates it from
the stream $\mathcal S$ through a closed-form continuous-time network~\cite{cfc}. With the arrival times and
gaps $\Delta_k$ of Def.~\ref{def:model}, and $x_k=\mathrm{enc}(o_k)$ the encoding of the $k$-th
observation, the state advances by
\begin{equation}
\label{eq:cfc}
h(t_k)=\;f_1(u_k)\odot\big(\mathbf 1-\sigma_k\big)\;+\;f_2(u_k)\odot\sigma_k,
\qquad
\sigma_k=\textstyle\mathrm{sig}\!\big(a(u_k)\,\Delta_k+b(u_k)\big),
\quad u_k=\phi\big([\,x_k;\,h(t_{k-1})\,]\big),
\end{equation}
where $\phi$ is a fixed backbone, $f_1,f_2$ are fixed maps, $a,b$ are fixed elapsed-time gates, and
$\sigma_k\in(0,1)^n$ interpolates per coordinate between the ``fast'' response $f_1$ and the ``slow''
response $f_2$ as a function of the real elapsed time $\Delta_k$. Equation~\eqref{eq:cfc} is the
closed-form CfC update~\cite{cfc}: the only time input is the elapsed gap, which is what makes the
summary continuous-time. A fixed read-out $r=\rho(h)$ maps the state to the agent's output.

\begin{remark}[Fixed substrate, moving state: the reservoir view]
\label{rem:fixed}
The parameters $(\phi,f_1,f_2,a,b)$ are \emph{fixed} at inference: what evolves is the \emph{state} $h$, not
the substrate. \emph{Fixed} means \emph{frozen during operation}, not \emph{arbitrary}: the weights are set
beforehand, by offline training or a spectral-radius-tuned initialization meeting the echo-state property
below, and that origin is outside this paper's scope. Offline-fixing sets the weights, not what the agent
knows: $h$ is acquired entirely in operation, so a trained-then-frozen substrate carries no pre-trained
knowledge, and the ``no ahead-of-time training'' claim is about the knowledge, not the substrate. The results
use only that the weights are frozen and form a stable summary basis, with no online update. From reservoir
computing and the Liquid State Machine~\cite{lsm,esn} we borrow one idea, that a fixed continuous-time
recurrence can carry computation in its dynamics, not their offline-trained readout ($\rho$ is fixed). Two
scopings keep this honest. First, the substrate is not a \emph{generic} reservoir: \S\ref{sec:necessity}
shows its dynamics must realise an adaptive timescale, so ``any rich recurrence will do'' does not apply.
Second, the necessity asks little of the particular weights, being a structural statement about the
substrate's \emph{form}, an adaptive continuous-time gate, that holds for any fixed weights realising it;
what \emph{does} need a weight condition is that the substrate be a useful summary at all, a stable fading
memory of recent input (the echo-state property~\cite{esn}), which we assume and do not claim for arbitrary,
possibly chaotic, weights. Adapting those weights online is the deeper frontier (\S\ref{sec:open}); the
estimation that runs \emph{on} the fixed substrate is online regardless.
\end{remark}

\section{Two necessities: an adaptive timescale and an elapsed-gap update}
\label{sec:necessity}
Section~\ref{sec:model} presented the substrate as a fixed-weight continuous-time recurrence; here we show
what it \emph{must} be, at the generality the claim needs. Optimal integration of a self-evolving latent,
observed asynchronously, requires two properties at once: an \emph{adaptive} (input-dependent) timescale and
a \emph{continuous-time} (elapsed-gap) update. Neither alone suffices: an LSTM forget gate is an adaptive
timescale that is not continuous-time, and a fixed continuous-time filter is timescale-rigid. We prove each
necessary here (Prop.~\ref{prop:necessity}, Thm~\ref{thm:sep}); the class realizing both, the \emph{liquid}
network, is defined once they are in hand (\S\ref{sec:liquid}). These are necessity claims, not sufficiency or
uniqueness, scoped to fixed-weight substrates (attainment is examined empirically in
\S\ref{sec:exp-liquid}); the full scope is in \S\ref{sec:open}.

\paragraph{The requirement.}
The agent integrates its observation stream into a state $h$ that estimates an evolving latent
$s(t)$: its current best estimate. Two demands sit on $h$ at once. \emph{Retention}: while $s$ is
stable, $h$ must reject observation noise and hold the accumulated value. \emph{Responsiveness}: when $s$
genuinely changes, the latent shifting to a new value, $h$ must move to it without long lag. An estimator of
a self-evolving latent lives in both regimes, often alternating; a substrate that fails either does not
follow it. We exhibit the tension concretely on the exponential filter, the form the CfC's gate takes, then
state the two necessities, which do not depend on that form and both rest on the latent and observation model
of Def.~\ref{def:model}.

\begin{lemma}[Noise--lag floor of a fixed gain]
\label{lem:floor}
Let the integrator be the fixed-gain update $h_k=(1-a)h_{k-1}+a\,o_k$, $a\in(0,1)$, with observations
$o_k=s_k+\varepsilon_k$ where $\varepsilon_k$ are zero-mean, variance $\sigma^2$, uncorrelated. Define its
\emph{noise} as the stationary error variance when $s$ is constant,
$V(a)=\lim_{k\to\infty}\operatorname{Var}(h_k-s)$, and its \emph{lag} as the integrated squared error of
the noiseless unit-step response (the state starting at the old value when $s$ jumps),
$L(a)=\sum_{k\ge0}\big(h_k-s_{\mathrm{new}}\big)^2$ for a unit jump. Then
\[
V(a)=\frac{a}{2-a}\,\sigma^2,\qquad
L(a)=\frac{1}{a(2-a)},\qquad
V(a)\,L(a)=\frac{\sigma^2}{(2-a)^2}\in\Big(\tfrac{\sigma^2}{4},\,\sigma^2\Big).
\]
Hence $V(a)\,L(a)>\sigma^2/4$ for every $a\in(0,1)$, with $\inf_a V(a)L(a)=\sigma^2/4$: the noise--lag
product is bounded below by a positive constant no fixed gain beats, so $V(a)\le\epsilon$ forces
$L(a)\ge\sigma^2/(4\epsilon)$.
\end{lemma}

\begin{proof}
\emph{Noise.} With $s$ constant the error $e_k=h_k-s$ obeys $e_k=(1-a)e_{k-1}+a\varepsilon_k$, a stable
AR(1) ($|1-a|<1$) driven by $a\varepsilon_k$; taking variances,
$\operatorname{Var}(e)=(1-a)^2\operatorname{Var}(e)+a^2\sigma^2$, so
$\operatorname{Var}(e)=\dfrac{a^2\sigma^2}{1-(1-a)^2}=\dfrac{a^2\sigma^2}{a(2-a)}=\dfrac{a}{2-a}\sigma^2$.
\emph{Lag.} For a unit step with the state at the old value, the noiseless response is $h_k=1-(1-a)^k$,
so the error is $h_k-1=-(1-a)^k$ and
$L(a)=\sum_{k\ge0}(1-a)^{2k}=\dfrac{1}{1-(1-a)^2}=\dfrac{1}{a(2-a)}$.
\emph{Product.} $V(a)L(a)=\dfrac{a}{2-a}\sigma^2\cdot\dfrac{1}{a(2-a)}=\dfrac{\sigma^2}{(2-a)^2}$. For
$a\in(0,1)$, $(2-a)^2\in(1,4)$, giving $V(a)L(a)\in(\sigma^2/4,\sigma^2)$ with infimum $\sigma^2/4$ as
$a\to0^+$. The stated implication is immediate.
\end{proof}

\begin{remark}[The product floor is exponential-filter-specific; the tension is general]
\label{rem:deadbeat}
The constant $\sigma^2/4$ is a feature of the one-pole form, not a universal law: a one-tap ``deadbeat''
filter $g_0{=}1$ (output the latest observation) follows a step with \emph{zero} lag but full noise
$V{=}\sigma^2$, so $V\!\cdot\!L{=}0$. What is general is the one-sided tension (low noise forces lag, since
an unbiased noise-suppressing impulse response with $\sum_k g_k{=}1$ and small $\sum_k g_k^2$ must spread
over $\gtrsim 1/\!\sum_k g_k^2$ taps, and a spread response settles slowly) and the optimal-filter statement
of Prop.~\ref{prop:necessity}, which holds for \emph{any} fixed filter, not just the one-pole.
\end{remark}

\begin{proposition}[Adaptivity is necessary]
\label{prop:necessity}
In the model of Def.~\ref{def:model} with change rate $\lambda>0$ (the latent stable for stretches, with
changes at unknown times), the minimum-mean-square-error (MMSE) estimate of $s(t)$ given the observations has
a \emph{time-varying} effective gain that no \emph{fixed}-gain (LTI) filter attains; every constant-gain
filter is therefore strictly suboptimal.
\end{proposition}

\begin{proof}[Proof]
The Bayesian filter for a change-point process maintains a posterior over the run length since the last
change~\cite{bocpd}; the optimal weight on a new observation is the posterior-mean Kalman gain, a function
of that run-length distribution. Just after a probable change the run-length posterior concentrates near
zero (high uncertainty), forcing a high gain; during a long stable run it concentrates at large run
length (low uncertainty), forcing a low gain~\cite{barshalom}. For $\lambda>0$ both regimes occur with
positive probability, so the optimal gain sequence is non-constant. The exact MMSE estimator is this
BOCPD posterior mean, a \emph{nonlinear} function of the observations (it reweights by the run-length
posterior), whereas every constant-gain choice is a linear time-invariant filter. For $\lambda>0$ the
run-length posterior is non-degenerate, so the conditional mean is genuinely nonlinear and no linear, hence
no constant-gain, filter can equal it (the conditional mean is the \emph{unique} MMSE estimator, so any
filter differing from it on a positive-probability set has strictly greater error); every constant choice
is therefore strictly suboptimal. (At $\lambda=0$, the degenerate stationary limit, the run length grows
without bound and the optimal gain settles to a constant, the steady-state Kalman/exponential filter, so
the hypothesis $\lambda>0$ is necessary; this is exactly the degenerate stationary limit of \S\ref{sec:exp-liquid}.)
The exponential-filter floor (Lemma~\ref{lem:floor}) is one explicit face of the resulting tension. An
input-dependent (adaptive) timescale is therefore necessary.
\end{proof}

The second leg strengthens from a suboptimality claim about one filter into a \emph{separation} between two
architecture classes, one that no amount of capacity can cross. Its mechanism is elementary, the law of
total variance applied to the gap variable, and its value is the architecture separation it yields, not its
difficulty. We first name the class precisely.

\begin{definition}[Gap-blind estimator]
\label{def:gapblind}
An estimator is \emph{gap-blind} if its estimate $\hat s_k$ is a measurable function of the observation
\emph{values} $O_k=(o_1,\dots,o_k)$ and the step index $k$ alone, not of the arrival times $t_{1:k}$ or the
elapsed gaps $\Delta_{1:k}$. This is exactly the per-observation unrolling of \emph{any} recurrence
$h_k=F_\theta(h_{k-1},o_k)$, $\hat s_k=g_\theta(h_k)$, with no time input: any width, depth, nonlinearity,
or trained weights $\theta$, so long as the gap is not supplied. A vanilla per-step RNN, LSTM, or GRU is
gap-blind; a continuous-time cell, or a discrete cell \emph{fed} $\Delta_k$ (Time-LSTM, GRU-D), is not.
Such a recurrence keeps state \emph{through steps}, not through elapsed time: it advances one update per
arrival and cannot tell a long gap from a short one.
\end{definition}

\begin{theorem}[Gap-awareness separation under irregular sampling]
\label{thm:sep}
Adopt the model of Def.~\ref{def:model} with $G$ strictly increasing (e.g.\ the random walk
$G(\Delta)=q\Delta$, or the change-point process of Prop.~\ref{prop:necessity}), sampled with
$\operatorname{Var}(\Delta)>0$ and noise $\sigma^2>0$. Write $s_k=s(t_k)$. Then
\emph{every} gap-blind estimator (Def.~\ref{def:gapblind}) obeys
\[
\mathbb E\big[(\hat s_k-s_k)^2\big]\;\ge\;\mathbb E\big[\operatorname{Var}(s_k\mid O_k)\big]
\;=\;\underbrace{\mathbb E\big[\operatorname{Var}(s_k\mid O_k,\Delta_{1:k})\big]}_{\text{gap-aware MMSE}}
\;+\;\mathcal I_k,
\]
where the nonnegative gap term is the variance reduction from observing the gaps,
\[
\mathcal I_k:=\mathbb E\big[\operatorname{Var}(s_k\mid O_k)\big]-\mathbb E\big[\operatorname{Var}(s_k\mid O_k,\Delta_{1:k})\big]\;\ge\;0 .
\]
The lower bound holds for every gap-blind estimator irrespective of its capacity or training. Moreover
$\mathcal I_k>0$ whenever the optimal estimate's gain depends on the gap (in particular for both
realizations of Def.~\ref{def:model} when $\operatorname{Var}(\Delta)>0$ and $\sigma^2>0$), while
$\mathcal I_k=0$ when sampling is regular ($\operatorname{Var}(\Delta)=0$). Hence under irregular sampling no
gap-blind architecture, at any size, attains the MMSE: matching it \emph{requires} the estimate to depend on the elapsed gap, and
network capacity does not substitute for that dependence.
\end{theorem}

\begin{proof}
The minimum mean-square estimate from a given information set is the conditional mean, with error the
conditional variance. A gap-blind estimator is, by Def.~\ref{def:gapblind}, a function of $O_k$ alone;
among all such functions the error is minimized by $\mathbb E[s_k\mid O_k]$, of error
$\mathbb E[\operatorname{Var}(s_k\mid O_k)]$. This optimum is attained only in the infinite-capacity limit,
so it lower-bounds every concrete gap-blind estimator (the first inequality), and the bound is
capacity-free precisely because it bounds the \emph{best} function of $O_k$, already granting the network
any gap information recoverable from the values themselves. Conditioning on the further variable
$\Delta_{1:k}$ cannot raise mean-square error (tower property / Rao--Blackwell), giving the decomposition
with $\mathcal I_k\ge0$. For strictness, the law of total variance (conditioning on $\Delta_{1:k}$ given
$O_k$) gives
\[
\mathcal I_k=\mathbb E\big[\operatorname{Var}_{\Delta_{1:k}\mid O_k}\!\big(\mathbb E[s_k\mid O_k,\Delta_{1:k}]\big)\big],
\]
so $\mathcal I_k>0$ iff the gap-aware posterior \emph{mean} $\mathbb E[s_k\mid O_k,\Delta_{1:k}]$ depends on
the gaps given the values, on a positive-probability set: it is the mean's gap-dependence, not the
variance's, that the decomposition isolates. This holds in both realizations of Def.~\ref{def:model}, because
the optimal weight on the latest observation is gap-dependent. For the Brownian drift the one-step gain
$K(\Delta_k)=(P+q\Delta_k)/(P+q\Delta_k+\sigma^2)$ is strictly increasing in $\Delta_k$, so the posterior
mean $(1-K)\,\hat s_{k-1}+K\,o_k$ moves with $\Delta_k$ whenever $o_k\neq\hat s_{k-1}$; and the values leave
the gap underdetermined ($\operatorname{Var}(\Delta_k\mid O_k)>0$, increment and noise confounded at
$\sigma^2>0$), so this mean-dependence survives conditioning on $O_k$. For the change-point process a longer gap raises the change
probability $1-e^{-\lambda\Delta_k}$, up-weighting the latest observation against the accumulated past, so
again the mean depends on $\Delta_k$. Hence $\mathcal I_k>0$ whenever the optimal gain depends on the gap.
When $\operatorname{Var}(\Delta)=0$ the gap is a known constant and adds nothing, so $\mathcal I_k=0$.
\end{proof}

\begin{remark}[How large the penalty is, and where it bites]
\label{rem:sep-size}
The one-step Gaussian case, the random-walk realization $G(\Delta)=q\Delta$ of Def.~\ref{def:model}, makes
$\mathcal I_k$ concrete. With prior variance $P$ before the gap, observing
$\Delta$ sets the predicted variance to $P+q\Delta$ and the posterior variance to
$v(\Delta)=\sigma^2(P+q\Delta)/(P+q\Delta+\sigma^2)$~\cite{jazwinski}; the gap-aware MMSE is
$\mathbb E[v(\Delta)]$. A gap-blind estimator must commit to a $\Delta$-independent posterior whose error is,
at best, $v$ evaluated at the value-implied estimate of $\Delta$. Since $v$ is strictly concave in $\Delta$
and $o_k$ identifies $\Delta$ only up to the noise floor $\sigma^2$, the gap is, to leading order in the
gap dispersion, $\mathcal I_k\approx\tfrac12\,|v''(\bar\Delta)|\,\operatorname{Var}(\Delta\mid O_k)>0$:
it grows with the dispersion of the sampling intervals and with the noise $\sigma^2$ (which both sharpens
the concavity and blurs the value-implied gap). The separation is therefore widest in the \emph{noisy,
irregularly-sampled} regime, exactly where asynchronous streams live, and closes only at high
SNR (the values reveal the gap) or under regular sampling (no gap to reveal).
\end{remark}

\begin{remark}[What the separation buys]
\label{rem:implications}
The theorem is a no-go with a constructive escape. \emph{No-go}: on a fixed-weight substrate over an
asynchronous stream, no gap-blind architecture, a vanilla RNN/LSTM/GRU however wide, deep, or long-trained,
reaches the optimum; the deficit is a missing \emph{input} (the elapsed gap), so added capacity cannot
remove it. \emph{Escape}: a single change closes it. Let the gap enter an input-dependent time-constant, the
liquid class of \S\ref{sec:liquid}, which is exactly why $\Delta$-fed cells (CfC/LTC, Time-LSTM, GRU-D)
succeed where their gap-blind ancestors fail. The engineering reading: in the noisy, unschedulable-timing
regime the lever is architecture, not scale; spend the budget on the gap channel.
\end{remark}

\section{The liquid class}
\label{sec:liquid}
The two necessities pick out one class of fixed-weight substrates: those realizing both an adaptive timescale
and an elapsed-gap update. We name it, give its realization, and locate the standard recurrences relative to
it.

\begin{definition}[Liquid substrate]
\label{def:liquid}
The substrate \eqref{eq:cfc} is a \emph{liquid substrate}: a fixed-weight (no online weight update)
continuous-time state recurrence whose update between successive observations has two properties, both of
which the two results above establish a self-evolving, asynchronously-observed substrate \emph{must} have.
\textnormal{(L1)} an \emph{adaptive timescale}: the effective integration time-constant is a function of the
input/state, so the estimator it realizes is time-varying (Prop.~\ref{prop:necessity}).
\textnormal{(L2)} \emph{elapsed-gap dependence}: the update is a function of the gap $\Delta_k$ since the
last observation, entering the dynamics as native structure rather than as a learned input (Thm~\ref{thm:sep}).
The gate $\sigma_k=\mathrm{sig}(a(u_k)\Delta_k+b(u_k))$ of \eqref{eq:cfc} keeps the timescale
input-dependent with frozen weights, and a multi-timescale state supplies (L1) without a hand-tuned rule.
The definition is the property pair, not a parameterization: the canonical realization is the MIT
liquid-network line, liquid time-constant networks~\cite{ltc} and their closed-form variant CfC~\cite{cfc},
but \emph{any} fixed-weight continuous-time cell meeting (L1)+(L2) is liquid in this sense (the exclusions are
itemized below). This paper's liquid substrate is therefore a \emph{liquid network of the MIT family at the
level of the cell, deployed in the reservoir / Liquid-State-Machine stance}
(Remark~\ref{rem:fixed}): the same dynamics in a different role, a fixed substrate, not a trained model.
\end{definition}

\begin{corollary}[The fixed-weight face: continuous-time is the native escape]
\label{cor:native}
Gap-awareness can be supplied two ways: \emph{structurally}, by a continuous-time cell whose update carries
$\Delta_k$ in its dynamics; or \emph{extrinsically}, by feeding $\Delta_k$ to a discrete cell trained to use
it (Time-LSTM, GRU-D). Among \emph{fixed-weight} substrates, with no training to the gap, only the structural
route applies, since the extrinsic one's compliance is learned. Thus on the reservoir / fixed-substrate
stance of Rem.~\ref{rem:fixed}, the elapsed-gap dependence Thm~\ref{thm:sep} requires is native to a
continuous-time (liquid) cell and unavailable to a frozen discrete one. Thm~\ref{thm:sep} forbids
gap-blindness at any capacity; the liquid cell of Def.~\ref{def:liquid} is how a frozen substrate escapes
the forbidden class.
\end{corollary}

\begin{corollary}[The mesh face: every agent must be liquid]
\label{cor:mesh}
Each agent of a mesh (Def.~\ref{def:mesh}) integrates an admitted stream that is exogenously timed,
non-stationary, and folded on a substrate fixed during operation, so each agent satisfies the hypotheses of
Def.~\ref{def:model} by construction. Prop.~\ref{prop:necessity} and Thm~\ref{thm:sep} therefore bind every
agent: a mesh's temporal substrate must be liquid at every node, with no node exempt and no center to carry
the dependence on its behalf. The statement is the per-agent necessity quantified over the population, not a
claim about the collective's dynamics: convergence of center-free inference, identification, and the
learning loop are the companion's domain~\cite{meshinf} and are not addressed here.
\end{corollary}

\noindent These two results are the two legs, both read off the single model of Def.~\ref{def:model},
each turning on a complementary facet of it, and together they close the gaps a single ``adaptive'' claim
leaves open. The latent's change rate forces an \emph{adaptive} timescale (Prop.~\ref{prop:necessity},
at full generality: a fixed filter is suboptimal, not only the one-pole form); irregular observation of the
same latent forces an \emph{elapsed-gap} update (Thm~\ref{thm:sep}: no gap-blind recurrence, at any
capacity, matches the MMSE). Their intersection is the liquid class (Def.~\ref{def:liquid}), and the
exclusions are clean. A
generic reservoir (echo-state network~\cite{esn}) has timescales set by its fixed leak rate and
weights, \emph{input-independent}, so it fails the adaptive leg. A \emph{vanilla, per-step} LSTM has an
input-dependent forget gate, adaptive, but updates once per observation regardless of elapsed time, so
it fails the elapsed-gap leg. A discrete net \emph{fed} the gap (GRU-D, Time-LSTM) attains both, but only
by being \emph{trained} to: a coerced member of the gap-aware class, outside the fixed-substrate
commitment. Among \emph{fixed-weight} recurrences, only a liquid one carries an adaptive, elapsed-gap
timescale without retraining: its gate \eqref{eq:cfc} stays input-dependent \emph{even with frozen weights},
exactly what committing to a fixed substrate buys rather than costs. Detecting change is the prerequisite of the adaptive leg~\cite{basseville,kalman,haykin},
which the multi-timescale realization below supplies.

\paragraph{The multi-timescale realization.}
The naive way to make the gain input-dependent, raise it when the incoming observation deviates from
$h$, fails, because a single observation's deviation cannot separate noise from change: in noise it
raises the gain and amplifies it (the negative control, \S\ref{sec:exp-liquid}). The liquid substrate's
actual structure resolves this. Its time constants are \emph{bimodal}, fast and slow neuron
populations~\cite{cfc}, so the state carries a fast track $h_f$ and a slow track $h_s$ at once; their
\emph{divergence} $|h_f-h_s|$ is the change signal, since noise does not persistently separate the two
while a real jump does. The effective gain rises with the divergence (trust the fast track at a change,
the slow track in noise), so adaptivity is supplied by the multi-timescale state itself, not by a
hand-tuned rule.

\section{Verification}
\label{sec:exp-liquid}
We verify each necessity in turn: the adaptive timescale of Prop.~\ref{prop:necessity} against baselines that
lack it (Fig.~\ref{fig:liquid}), and the gap-awareness separation of Thm~\ref{thm:sep} as an exact,
training-free computation of the bound itself (Fig.~\ref{fig:separation}).

\subsection{The adaptive timescale: liquid versus fixed}
\label{sec:exp-prop}
Figure~\ref{fig:liquid} tests Prop.~\ref{prop:necessity} directly. A scalar agent integrates an
asynchronous, bursty, noisy stream of an evolving latent (piecewise-constant with occasional jumps);
we compare fixed-fast, fixed-slow, the naive single-deviation liquid rule, and the multi-timescale
liquid integrator (seeded, reproducible; \texttt{liquid\_repro.py}). (a)~Fixed-fast follows jumps but is
noisy, fixed-slow is smooth but lags, the multi-timescale liquid does both. (b)~On the
noise/lag plane the two fixed timescales sit on the tradeoff frontier Prop.~\ref{prop:necessity}
describes; the naive liquid rule sits beside fixed-fast (it mistakes noise for change); and of the
integrators tested only the multi-timescale liquid reaches the low-noise, low-lag corner; no fixed
timescale does, as Prop.~\ref{prop:necessity} requires. (c)~Its joint noise-and-lag advantage shows up as the best in-band
fraction ($99\%$ vs.\ $96$--$98\%$): it stays within a fixed band of the truth most often, though the
margin over fixed-slow ($98\%$) is small, so the in-band claim rests on the joint (a,b) advantage, not panel
(c) alone. Across stream regimes the multi-timescale liquid has the lowest error
wherever the latent \emph{evolves} (calm, volatile, storm); only in the degenerate stationary
limit, pure noise around a constant with nothing to evolve and hence nothing to integrate, does a fixed slow average
match or beat it. That limit is the boundary of the claim, not a counterexample: adaptivity is necessary
exactly where there is something to adapt to.
\begin{figure}[h]\centering
\includegraphics[width=\linewidth]{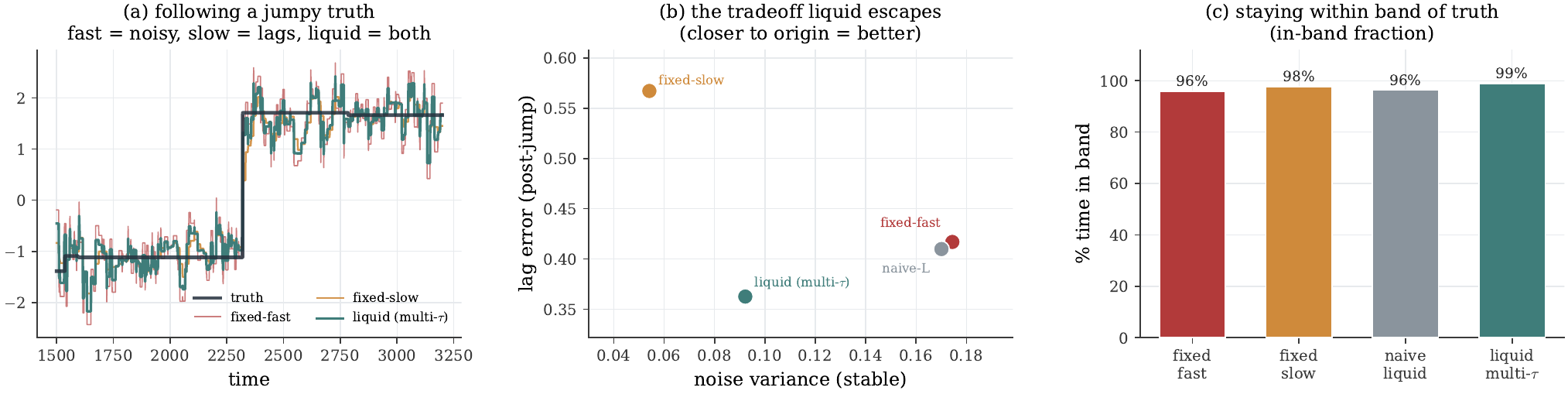}
\caption{An adaptive timescale is necessary (synthetic, reproducible). A scalar agent integrates an evolving
latent from a noisy, bursty, asynchronous stream. (a)~fixed-fast is noisy, fixed-slow lags, the
multi-timescale liquid does both. (b)~the fixed timescales sit on the noise/lag tradeoff frontier and the
naive single-deviation liquid rule fails (it amplifies noise); only the multi-timescale liquid (fast/slow
divergence as the change signal) reaches the low-noise, low-lag corner. (c)~the liquid integrator stays
within a fixed band of the truth the most, though the margin over fixed-slow is small; the in-band claim
rests on the joint (a,b) advantage, not (c) alone.}
\label{fig:liquid}
\end{figure}

\subsection{The gap-awareness separation}
\label{sec:exp-sep}
Figure~\ref{fig:separation} tests Thm~\ref{thm:sep}. Because the bound is informational, not architectural, it
needs no training: in the one-step Gaussian case (the random-walk realization of Def.~\ref{def:model}), the
gap-aware MMSE $\mathbb E[v(\Delta)]$ and the \emph{best} gap-blind error $\mathbb E[\operatorname{Var}(s\mid
O)]$, the floor every gap-blind estimator obeys at any width or depth, are computed exactly by integrating out
the random gap. (a)~The gap-blind error stays strictly above the gap-aware floor, and the shortfall widens
with the dispersion of the arrival times. (b)~The separation $\mathcal I_k$ is zero under regular sampling
($\operatorname{Var}(\Delta){=}0$) and rises with $\operatorname{Var}(\Delta)$, tangent at the origin to the
analytic leading-order law of Remark~\ref{rem:sep-size} (the coefficient $\tfrac12|v''(\bar\Delta)|$, not a fit). (c)~It is widest in the noisy regime and closes at high SNR,
where the values themselves reveal the gap, exactly as Remark~\ref{rem:sep-size} predicts. This is the gap no
gap-blind network closes with scale; \S\ref{sec:open} turns it into the dispersion sweep a deployed mesh would
measure.
\begin{figure}[h]\centering
\includegraphics[width=\linewidth]{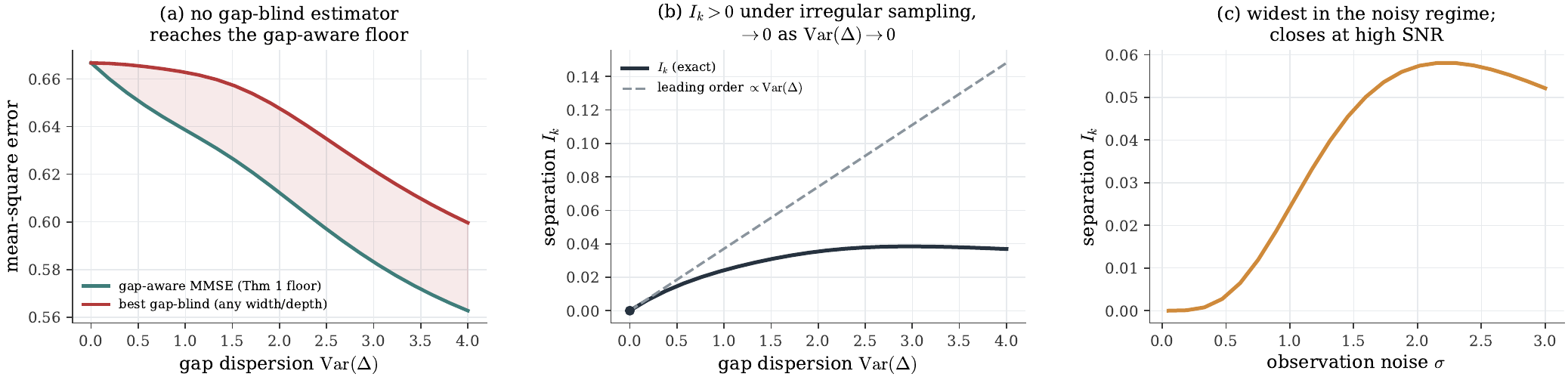}
\caption{The gap-awareness separation is real and grows with irregularity (exact, training-free). One-step
Gaussian model (Def.~\ref{def:model}, random walk); the gap-aware MMSE and the \emph{best} gap-blind error are
computed by integrating out the random gap, so the bound holds at any capacity. (a)~the best gap-blind error
stays above the gap-aware floor, the shortfall widening with $\operatorname{Var}(\Delta)$. (b)~the separation
$\mathcal I_k$ vanishes at $\operatorname{Var}(\Delta){=}0$ and rises with it, matching the analytic
leading-order law (Rem.~\ref{rem:sep-size}) near the origin. (c)~$\mathcal I_k$ peaks in the noisy regime and closes at high
SNR.}
\label{fig:separation}
\end{figure}

\section{Limitations and future directions}
\label{sec:open}
\paragraph{Assumptions.} The results rest on assumptions whose violation changes them. \emph{Exogenous
sampling}: arrivals are state-independent, so the elapsed gap is non-informative about the latent's value and
acts only through the gain (Thm~\ref{thm:sep}); informative, state-modulated sampling would let the gap carry
value too, and the decomposition would not hold. \emph{A self-evolving latent}: Prop.~\ref{prop:necessity} is
proven for a piecewise-stationary (change-point) latent and Thm~\ref{thm:sep} for any strictly-increasing
$G(\Delta)$ (a Brownian drift included); other non-stationarity classes, and the multivariate,
partially-observed case, are open. \emph{A single agent}: the theorems are per-agent necessities, lifted to
every node of a mesh by Cor.~\ref{cor:mesh}; they do not address the collective's multi-agent dynamics, such
as the convergence of center-free inference~\cite{meshinf}.

\paragraph{What is not claimed.} The result is necessity, not sufficiency: that a liquid substrate
\emph{attains} the optimum is examined only empirically, on a synthetic benchmark (\S\ref{sec:exp-prop}), and
is not proven; the separation of Thm~\ref{thm:sep}, by contrast, is computed exactly (\S\ref{sec:exp-sep}).
Real-data validation remains future work. The separation is informative only where it
bites: its shortfall grows with the dispersion of the arrival times (Remark~\ref{rem:sep-size}) and is
negligible under near-regular sampling.

\paragraph{Future directions.} Relaxing the fixed-weight premise is the bridge from inference to learning.
What an adapting substrate should move \emph{toward} is, in a mesh, the collective's intention, the
\emph{ask}, which reaches each agent the way its observations do; whether such ask-driven adaptation is
stable rather than drifting is open. Beyond the substrate, further questions remain open: the
content gate that admits what is relevant~\cite{svaf}, resilient admission under adversaries, and what
\emph{grounds} the collective rather than merely keeping it coherent. Figure~\ref{fig:separation} computes that
separation exactly in the Gaussian model; carrying it to real data is the next step. A deployed edge mesh,
whose agents integrate genuinely irregular peer streams, supplies the test in situ: the
gap-blind-versus-liquid tracking gap should widen with the measured dispersion of inter-arrival times, the
real-data validation the synthetic computation stands in for.

\section{Conclusion}
Two facts about integrating a self-evolving latent online determine what the substrate must be: the latent
changes, and its observations arrive at clock-free times. The first forces a time-varying optimal estimator
(Prop.~\ref{prop:necessity}); the second forces an elapsed-gap gain that no gap-blind recurrence, at any
capacity, can match (Thm~\ref{thm:sep}). Among fixed-weight substrates, the intersection of the two is the
liquid class: an LSTM has the first property only and a fixed continuous-time filter the second only, while a
multi-timescale liquid network supplies both, outperforming the fixed-timescale and naive-adaptive
baselines, with the gap-awareness separation it exploits computed exactly (\S\ref{sec:exp-liquid}). The characterization is scoped, not universal: it binds optimal integration on
a fixed-weight substrate, not a learner free to retrain its weights, and weight-adapting recurrences attain
both properties by other means (Remark~\ref{rem:fixed}). The setting that makes the necessity bite is
\emph{mesh intelligence}: there the quantities an agent tracks keep moving, its arrivals cannot be
scheduled, and its weights are fixed, so the conditions of both results hold by construction, and this
liquid substrate is the temporal layer of its coupling. Whether that substrate can go further and adapt its
own weights online, beyond the fixed-weight regime characterized here, is the direction this work opens.

\end{document}